\documentclass[letterpaper, 10 pt, conference]{ieeeconf}  

\IEEEoverridecommandlockouts                              

\usepackage{amsmath,amsfonts,amssymb}
\usepackage{algorithm}
\usepackage[noend]{algorithmic} 
\usepackage{array}
\usepackage[caption=false,font=normalsize,labelfont=sf,textfont=sf]{subfig}
\usepackage{textcomp}
\usepackage{stfloats}
\usepackage{url}
\usepackage{hyperref}
\usepackage{verbatim}
\usepackage{graphicx}
\usepackage{cite}
\usepackage{optidef} 

\def\bbr{\mathbb R}

\def\bbs{\mathbb S}
\def\bbp{\mathbb P}
\def\bbe{\mathbb E}

\def\calx{\mathcal X}

\def\calu{\mathcal U}

\def\calo{\mathcal O}

\def\calp{\mathcal P}

\newcommand{\PP}{\mathbb{P}}

\newcommand{\norm}[1]{\left\|#1 \right\|}

\usepackage{amsthm} 
\newtheorem{theorem}{Theorem} 
\newtheorem{lemma}{Lemma} 
\newtheorem{problem}{Problem} 
\usepackage{float} 
\usepackage{multicol} 

\title{\LARGE \bf
Distributionally Robust RRT with Risk Allocation
}

\author{Kajsa~Ekenberg, Venkatraman~Renganathan, and
Björn Olofsson
\thanks{This project has received funding from the European Research Council (ERC) under the European Union’s Horizon 2020 research and innovation program under grant agreement No 834142 (Scalable Control). K. Ekenberg was a Master's Thesis Student of the Department of Automatic Control LTH, Lund University, Sweden. V. Renganathan and B. Olofsson are with the Department of Automatic Control LTH, Lund University, Lund, Sweden. E-mail: ka2560ek-s@student.lu.se, (venkat, bjorn.olofsson)@control.lth.se.}
}

\begin{document}

\maketitle
\thispagestyle{empty}
\pagestyle{empty}

\begin{abstract}

An integration of distributionally robust risk allocation into sampling-based motion planning algorithms for robots operating in uncertain environments is proposed. We perform non-uniform risk allocation by decomposing the distributionally robust joint risk constraints defined over the entire planning horizon into individual risk constraints given the total risk budget. Specifically, the deterministic tightening defined using the individual risk constraints is leveraged to define our proposed exact risk allocation procedure. Embedding the risk allocation technique into sampling-based motion planning algorithms realises guaranteed conservative, yet increasingly more risk-feasible trajectories for efficient state-space exploration. 

\end{abstract}

\section{Introduction}
Motion planning under uncertainty becomes challenging when only limited information about the system uncertainty is known. Such lack of information adds complexity to the existing path planning problem formulation for finding guarantees on the safety of the path generated by algorithms that aim to address such uncertainty. For instance, critical robotics operations such as Mars Rover and rescue robot missions cannot afford the risk of obstacle collision given their high stakes and environmental uncertainty. Often assumptions (such as Gaussian uncertainties) are made in the name of tractability as in \cite{luders_ccrrt, luders_ccrrtstar} and they may lead to significant miscalculation of risk. Recently, this shortcoming due to non-Gaussian stochastic uncertainties in motion planning was considered in \cite{summers_risk, venki_risk, han_risk, lathrop2021distributionally, luders2011probabilistic} using approaches like distributionally robust optimization (DRO) \cite{mohajerin2018data, hota2019data} and conditional value-at-risk \cite{Astghik_Cvar}. 

Though many risk-bounded path planning techniques work with stochastic uncertainties characterised by their moments, they suffer from unwanted conservatism as a result of the uniform risk allocation (URA) being used. That is, given a total risk budget for safety violation, it is a common practice to distribute it uniformly across all the obstacles and the planning horizon. The conservatism drawback of URA was identified in \cite{ono_iter_risk} and approached using a two-stage optimization method based iterative risk allocation strategy. This strategy has yielded promising, less conservative and guaranteed results for covariance steering problems, for example in spacecraft maneuvering as in \cite{pilipovsky2021covariance, renganathancovsteer}. 

The conservatism that arises in motion planning because of the lack of exact information about the stochastic uncertainties should not restrict the ability of sampling-based algorithms like RRT \cite{lavalle1998rapidly} to efficiently explore the state-space too much. To overcome this shortcoming, one can use non-uniform risk allocation technique as in \cite{ono_planning, vitus_risk_feedback, kajsa_ms_thesis}, so that they do not exceed the risks allocated uniformly. 

\emph{Contributions:} We extend the DR-RRT algorithm presented in \cite{summers_risk} by embedding our proposed risk allocation technique into it. Our main contributions are as follows:
\begin{enumerate}
    \item We propose a new distributionally robust risk allocation technique called Exact Risk Allocation (ERA) for sampling-based motion planning algorithms that allocates as minimum risks as possible while respecting a given total risk budget (See Theorem \ref{thm_1}).
    \item We prove that all feasible paths with the uniform risk allocation of length $\mathrm{\mathrm{T_{path}}} \in \mathbb{N}_{\geq 1}$ and total risk budget $\mathrm{\mathrm{\Delta_{path}}} \in (0,0.5]$ is also feasible with the ERA but the vice-versa is not necessarily true (See Theorem \ref{thm_2}). 
    \item We demonstrate our proposed technique using simulation results and show that by switching from uniform risk allocation to ERA, it is possible to give the same risk guarantees for sampling-based motion planning algorithms while maintaining a reduced conservatism.
\end{enumerate}

The rest of the paper is organized as follows: The main problem statement of risk-bounded motion planning with risk allocation is presented in \S \ref{sec_prob_formulation}. Then, the proposed Distributionally Robust Exact Risk Allocation (DR-ERA) algorithm is presented in \S \ref{sec_dr_era}. Subsequently, the embedding of DR-ERA into the sampling-based motion planning algorithm RRT is discussed in \S\ref{sec_dr_rrt_era}. Then, the proposed idea is demonstrated using simulation results in \S \ref{sec_num_sim}. Finally, the paper is closed in \S \ref{sec_conclusion} along with the directions for future research.

\section*{Notations \& Preliminaries}
The set of real numbers and natural numbers are denoted by $\mathbb{R}$ and $\mathbb{N}$, respectively. The subset of natural numbers between and including $a$ and $b$ with $a < b$ and beyond $b$ with $b$ included are denoted by $[a:b]$ and $\mathbb{N}_{\geq b}$, respectively. The operators $\oplus, \backslash$, and $|\cdot|$ denote the set translation, set subtraction and set cardinality, respectively. The transpose of a matrix $A$ is denoted by $A^{\mathsf{T}}$. An identity matrix of dimension $n$ is denoted by $I_{n}$. For a non-zero vector $x \in \bbr^{n}$ and a matrix $P \in \bbs^{n}_{++}$ (here, $\bbs^{n}_{++}$ denotes the set of positive definite matrices), let $\left \| x \right \Vert_{P} = \sqrt{x^{\mathsf{T}} P x}$. A binary condition being true and false is denoted by $\top$ and $\bot$, respectively. 

\section{Problem Formulation} \label{sec_prob_formulation}
\subsection{Robot \& Environment Model}
Our problem formulation follows the problem setup given in \cite{summers_risk}. Consider a robot operating in an uncertain environment, $\mathcal{X} \subseteq \bbr^n$, with dynamic obstacles. The set of obstacles is denoted as $\mathcal{B}$ with $\left | \mathcal{B} \right \vert = N$. The robot model is given by a stochastic discrete-time linear time invariant system  
\begin{equation} \label{eqn_robot_dynamics}
x_{k+1}= Ax_k + Bu_k + w_k,
\end{equation}
where $x_{k} \in \bbr^n$ and $u_{k} \in \bbr^m$ are the system state and input at time step $k$, respectively. The matrices $A$ and $B$ denote the dynamics matrix and the input matrix, respectively. The process noise $w_k \in \bbr^m$ is a zero-mean random vector that is independent and identically distributed over time. The distribution of $w_k$, namely $\mathbb{P}_{w_{k}}$, is unknown but belongs to a moment-based ambiguity set of distributions, 
\begin{equation}
\calp^w = \left\{ \PP_{w_{k}} \mid \bbe[w_k]=0, \bbe[w_k w_k^{\mathsf{T}}] = \Sigma_{w} \right\}.
\end{equation}
The initial state $x_{0}$ is subject to a similar uncertainty model as the process noise, with its distribution belonging to a moment-based ambiguity set, $\mathbb{P}_{x_{0}} \in \calp^{x_0}$, given by
\begin{equation}
\footnotesize
    \calp^{x_0}= \left\{ \PP_{x_{0}} \mid \bbe[x_0]=\hat{x}_0, \bbe[(x_0-\hat{x}_0) (x_0-\hat{x}_0)^{\mathsf{T}}] = \Sigma_{x_0} \right\}.
\end{equation}
We assume the obstacles to perform a random walk around their initial position. That is, 
\begin{equation} \label{eqn_obs_set_dynamics}
\mathcal{O}_{ik}=\mathcal{O}_{i}^{0} \oplus \hat{c}_{ik} \oplus c_{ik}, \quad \forall i \in \mathcal{B},
\end{equation}
where $\mathcal{O}_{ik}$ denotes the position of the obstacle $i \in \mathcal{B}$ at time step $k$. The known shape of the obstacle is represented by  $\mathcal{O}_i^0 \subset \bbr^n$, while $\hat{c}_{ik}$ represents a known nominal translation. The location uncertainty and unpredictable motion of  obstacle $i \in \mathcal{B}$ is represented by $c_{ik} \in \bbr^n$, which is a random vector with unknown distribution  $\PP_{ik}^c \in \calp_{ik}^c$. The robot is expected to be in the free space at all time steps $k$. That is, 
\begin{align} \label{eqn_x_free_dyn}
x_{k} \in \mathcal{X}^{\mathrm{free}}_{k} := \mathcal{X} \setminus \bigcup_{i \in \mathcal{B}} \mathcal{O}_{ik},
\end{align}
and the input of the robot is subject to the constraint $u_{k} \in \mathcal{U}$. Here, $\mathcal{U}$, $\mathcal{X}$ and $\mathcal{O}_{ik}$ are all assumed to be convex polytopes that can be represented by a conjunction of linear inequalities  
\begin{align}
    \calu &= \left\{ u_k \mid A_{u} u_k \le b_{u} \right\}, \label{eqn_controlconstraint_polytope}\\
    \calx &= \left\{ x_k \mid A_{x} x_k \le b_{x} \right\}, \label{eqn_environment_polytope}\\
    \calo_{ik} &= \left \{ x_k \mid A_{ik} x_k \le b_{ik} \right\}. \label{eqn_obstacle_polytope}
\end{align}

\subsection{Distributionally Robust Path Planning Problem}
\begin{problem} \label{problem_1}
We seek to approximately solve the distributionally robust risk-constrained path planning problem. Given an uncertain initial state $x_0 \sim \mathbb{P}_{x_{0}}$ and a set of goal locations $\mathcal{X}_{\mathrm{goal}} \subset \bbr^n$, we seek to find a feedback  control policy $\pi = \{\pi_k\}^{T-1}_{k=0}$ such that applying the control inputs $u_k = \pi_k(x_k), k = [0:T-1]$ yields a probabilistically feasible path from the initial state to the goal that minimises a finite-horizon cost function. That is, 
\begin{subequations} \label{eqn_problem_1}
\begin{alignat}{2}
&\! \underset{\pi}{\mathrm{minimize}}        &\qquad& \sum_{k=0}^{T-1} \ell_t(\hat{x}_k, \calx_{\mathrm{goal}}, u_t)+ \ell_T(\hat{x}_T, \calx_{\mathrm{goal}}) \label{eq:optProb}\\
&\mathrm{subject} \, \mathrm{to} &      & \eqref{eqn_robot_dynamics}, x_0 \sim \mathbb{P}_{x_0} \in \calp^x, w_{k} \sim \mathbb{P}_{w} \in \calp^w,\label{eq:constraint1}\\
&                  &      & \eqref{eqn_obs_set_dynamics}, \eqref{eqn_x_free_dyn}, \, u_k \in \calu, \, c_{ik} \sim \PP_{ik}^c \in \calp_{ik}^c, \label{eq:constraint2} \\
&                  &      & \sup_{\bbp_{x_k} \in \calp^{x_{k}} } \bbp_{x_k} \left [\bigwedge_{k=1}^{T} x_{k} \notin \mathcal{X}^{\mathrm{free}}_{k}\right] \leq \Delta. \label{cc_joint}
\end{alignat}
\end{subequations}
\end{problem}

Problem 1 is just a reformulated version of the problem in \cite{summers_risk}. Here, $\ell_t(.)$ is the stage cost function quantifying the distance to the goal set and actuator effort and it is expressed in terms of the robot mean state, $\hat{x}_k$. As \eqref{cc_joint} is an infinite dimensional DR risk constraint, solving \eqref{eqn_problem_1} exactly is practically hard and so we resort to approximate solutions using sampling-based motion planning algorithms. The constant $\Delta \in (0,0.5]$ represents the user-prescribed total risk budget for the entire planning horizon, such that the worst-case probability of colliding with any of the $N$ obstacles or being outside $\calx$ over the planning horizon should be at most $\Delta$. As in \cite{summers_risk}, an LQR fixed affine feedback control policy given by $u_{k} = K_{k} x_{k} + g_{k}$ is used, and the state mean $\hat{x}_{k}$ and covariance matrix $\Sigma_{x_{k}}$ evolve as 
\begin{align}
\label{mean_prop}
    \hat{x}_{k+1} &= (A + BK_k)\hat{x}_k + Bg_k, \\
    \Sigma_{x_{k+1}} &= (A+BK_k)\Sigma_{x_k}(A+BK_k)^\mathsf{T} + \Sigma_{w}. \label{cov_prop}
\end{align}
Note that \eqref{cc_joint} can be decomposed into individual chance constraints for each obstacle and the state constraint $\mathcal{X}$ at each time step. The individual risk bound for each obstacle $i \in \mathcal{B}$ and the constraints $j = 1,\dots,n_{e}$ defining $\mathcal{X}$ at time step $k$, denoted by $\delta_{ik}$ and $\kappa_{jk}$ respectively, should respect 
\begin{align} \label{eqn_risk_sum}
\sum_{k=1}^T \sum_{i=1}^N \delta_{ik}  + \sum_{k=1}^T \sum_{j=1}^{n_{e}} \kappa_{jk} \leq \Delta. 
\end{align}
The following lemma is an adaptation of Theorem 1 in \cite{summers_risk} with inclusion of time horizon from $t = 1, \dots,T$.
\begin{lemma}
If \eqref{eqn_risk_sum} holds true, then \eqref{cc_joint} holds true as well if the worst-case probability of colliding with obstacle $i$ and the worst-case probability of violating any one of $j = 1,\dots, n_{e}$ constraints defining $\mathcal{X}$ at time step $k \in [1:T]$ are 
\begin{align} 
    \sup_{\bbp_{x_k} \in \calp^{x_{k}} } \bbp_{x_k} ( x_k \in \mathcal{O}_{ik}) &\leq \delta_{ik}, \quad \mathrm{and},\label{eqn_prob_xk_in_obs_i} \\
    \sup_{\bbp_{x_k} \in \calp^{x} } \bbp_{x_k} ( a^{\mathsf{T}}_{xj} x_{k} \geq a^{\mathsf{T}}_{xj} c_{xj}) &\leq \kappa_{jk}. \label{eqn_prob_xk_outside_X_j}
\end{align}
\end{lemma}

\begin{proof}
We know that $x_k \notin \mathcal{X}^{\mathrm{free}}_{k} \iff \left\{ x_k \in \bigcup_{i=1}^N \calo_{ik} \right\} \cup \{ x_k \notin \mathcal{X} \}$. We denote the event of colliding with obstacle $i$ at time step $k$ as $C_{ik} := x_k \in \calo_{ik}$ and similarly let the event of violating the $j$\textsuperscript{th} constraint defining the state constraint set $\mathcal{X}$ at time step $k$ be $D_{jk} := \left\{ a^{\mathsf{T}}_{xj} x_{k} \geq a^{\mathsf{T}}_{xj} c_{xj} \right \}$. Then, the left-hand side of \eqref{cc_joint} can be equivalently written as
\begin{align*}
    &\sup_{\bbp_{x_k} \in \calp^{x_k} } \bbp_{x_k} \left [\bigvee_{k=1}^T \left [ \left\{ x_k \in \bigcup_{i=1}^N \calo_{ik} \right\} \cup \{ x_k \notin \mathcal{X} \} \right] \right] \\
    &\leq \sup_{\bbp_{x_k} \in \calp^{x_k} } \bbp_{x_k} \left [\bigvee_{k=1}^T \bigvee_{i=1}^N C_{ik} \right] + 
    \sup_{\bbp_{x_k} \in \calp^{x_k} } \bbp_{x_k} \left [\bigvee_{k=1}^{T} \bigvee_{j=1}^{n_{e}} D_{jk} \right] \\
    &\leq \sum_{k=1}^{T} \sum_{i=1}^{N} \sup_{\bbp_{x_k} \in \calp^{x_{k}}} \bbp_{x_k} [C_{ik}] + \sum_{k=1}^{T} \sum_{j=1}^{n_{e}} \sup_{\bbp_{x_k} \in \calp^{x}} \bbp_{x_k} [D_{jk}] \\
    &\leq \sum_{k=1}^{T} \sum_{i=1}^{N} \delta_{ik} + \sum_{k=1}^T \sum_{j=1}^{n_{e}} \kappa_{jk} \\
    &\leq \Delta.
\end{align*}
Here, we applied the Boole's inequality \cite{hunter_book} to get the second inequality, \eqref{eqn_prob_xk_in_obs_i} and \eqref{eqn_prob_xk_outside_X_j} to get the third inequality and \eqref{eqn_risk_sum} for the fourth inequality to obtain the desired result.
\end{proof}
\noindent Future research will seek to reduce the conservatism resulting from the Boole's inequality by using sharper bounds such as the Kwerel's, Kounias' or Hunter's bounds \cite{kwerel1975bounds, patil2021upper}. We now reformulate Problem 1 with individual risk bounds.
\begin{problem}
We seek to approximately solve the following distributionally robust path planning problem with individual risk bounds as follows:
\begin{mini!}|l|[2]
{\pi, \delta}{\sum_{k=0}^{T-1} \ell_t(\hat{x}_k, \calx_{\mathrm{goal}}, u_t)+ \ell_T(\hat{x}_T, \calx_{\mathrm{goal}})}
{}{}
\addConstraint{\eqref{eqn_robot_dynamics}, x_0 \sim \mathbb{P}_{x_0} \in \calp^x, w_{k} \sim \mathbb{P}_{w} \in \calp^w}
\addConstraint{\eqref{eqn_obs_set_dynamics}, \eqref{eqn_x_free_dyn}, \, u_k \in \calu, \, c_{ik} \sim \PP_{ik}^c \in \calp_{ik}^c}
\addConstraint{\sup_{\bbp_{x_k} \in \calp^{x} } \bbp_{x_k} ( x_k \in \mathcal{O}_{ik}) \leq \delta_{ik}, \, \substack{\forall i \in \mathcal{B}, \\ \forall k \in [1:T]}} \label{cc_individual}
\addConstraint{\sup_{\bbp_{x_k} \in \calp^{x} } \bbp_{x_k} ( a^{\mathsf{T}}_{xj} x_{k} \geq a^{\mathsf{T}}_{xj} c_{xj}) \leq \kappa_{jk}, \substack{\forall j \in [1:n_{e}], \\ \forall k \in [1:T]} } \label{cc_individual_env}
\addConstraint{\sum_{k=1}^T \sum_{i=1}^N \delta_{ik} + \sum_{k=1}^T \sum_{j=1}^{n_{e}} \kappa_{jk} \leq \Delta. } \label{sumsum}{} 
\end{mini!}
\end{problem}
\noindent The only difference between Problems 1 and 2 is that Problem 2 is expressed with individual risk constraints and the allocated individual risks satisfy the total risk budget $\Delta$. 

\section{Distributionally Robust Risk Allocation} \label{sec_dr_era}
Allocating the individual risks in a non-uniform way while still enforcing the DR risk constraint \eqref{cc_joint} can minimise the conservatism of the resulting path from source to the goal. Let us define the vector of all individual risk bounds as 
\begin{align}
\delta := \begin{bmatrix}\delta_{11} & \dots & \delta_{NT} \end{bmatrix}^{\mathsf{T}}, \kappa := \begin{bmatrix}\kappa_{11} & \dots & \kappa_{n_{e}T} \end{bmatrix}^{\mathsf{T}}.
\end{align} 

\subsection{Risk Treatment: Polytopic Obstacles \& State Constraints}
Since the obstacle $\mathcal{O}_{ik}, \forall i \in \mathcal{B}$ is a convex polytope, it can be represented by $n_i$ hyperplanes. Collision with obstacle $i \in \mathcal{B}$ at time step $k$ occurs if the position of the robot lies inside the obstacle, $x_k \in \calo_{ik}$. This can be expressed as a conjunction of $n_i$ linear constraints on the robot's position, 
\begin{equation}
x_k \in \mathcal{O}_{ik} \quad \iff \quad \bigwedge_{j=0}^{n_i} a^{\mathsf{T}}_{ij} x_{k} < b_{ikj}. 
\end{equation} 
The individual chance constraints given by \eqref{eqn_prob_xk_in_obs_i} encode the fact that the worst-case probability of colliding with obstacle $i$ at time step $k$ should be at most $\delta_{ik}$. That is,
\begin{align}
\sup_{\bbp_{x} \in \calp^{x_{k}}} \bbp_{x_k}  \left [\bigwedge_{j=1}^{n_i} a^{\mathsf{T}}_{ij} x_{k} < a^{\mathsf{T}}_{ij} c_{ikj} \right] \leq \delta_{ik}, \label{cc_obs}
\end{align}
where $c_{ikj}= \hat{c}_{ikj} + c_{ik}$ is a point on the $j$th constraint of obstacle $\calo_{ik}$, with its first and second moments being $\hat{c}_{ikj}$ and $\Sigma_{cjk}$ respectively. The distributionally robust individual risk constraint in \eqref{cc_obs} can be handled through linear constraints on the state mean $\hat{x}_k$ defined using deterministic constraint tightening as in \cite{calafiore2006distributionally, summers_risk}. That is,
\begin{align}
    a^{\mathsf{T}}_{ikj} \hat{x}_{k} &\geq a^{\mathsf{T}}_{ikj} \hat{c}_{ikj} + \gamma^{j}_{ik}, \label{eqn_dr_constraint_tightening}
    \\
    \gamma^{j}_{ik}(\delta_{ik}) &:= \sqrt{ \frac{1-\delta_{ik}}{\delta_{ik}}} \norm{(\Sigma_{x_k} + \Sigma_{cjk})^{\frac{1}{2}}a_{ikj}}_{2}, \label{tight} 
\end{align}
where, $\gamma^{j}_{ik}$ is the deterministic constraint tightening of the $j$\textsuperscript{th} constraint of obstacle $i \in \mathcal{B}$ at time $k$. To this end, we define Boolean quantities $\mathbf{h}^{j}_{ik}$ and $\mathbf{h}_{ik}$ that represent the mean state being outside the tightened $j$th constraint of $\calo_{ik}$ and outside the tightened obstacle $\calo_{ik}$, respectively:
\begin{align}
    \mathbf{h}^{j}_{ik} &= 
    \begin{cases}
    \top, \ \eqref{eqn_dr_constraint_tightening} \text{ is satisfied} \\ 
    \bot, \ \text{otherwise},
    \end{cases} \label{eqn_hjik_boolean}\\
    \mathbf{h}_{ik} &= 
    \begin{cases}
    \top, \ \ \bigvee_{j=1}^{n_i} \mathbf{h}^{j}_{ik} = \top \\
    \bot, \ \text{otherwise}.
    \end{cases} \label{eqn_hik_boolean}
\end{align}
Here, \eqref{eqn_dr_constraint_tightening} encodes the condition that the mean position of the robot should lie outside the tightened $j$\textsuperscript{th} constraint of obstacle $i \in \mathcal{B}$ at time $k$ to fulfill $\mathbf{h}^{j}_{ik} = \top$. A similar approach can be taken for treating the state constraints. The distributionally robust individual risk constraint in \eqref{cc_individual_env} can be handled by linear constraints on the state mean $\hat{x}_k$ defined using deterministic constraint tightening: 
\begin{align} \label{eqn_dr_constraint_tighten_env}
    a^{\mathsf{T}}_{xj} \hat{x}_{k} \leq a^{\mathsf{T}}_{xj} c_{xj} - \underbrace{\sqrt{ \frac{1-\kappa_{jk}}{\kappa_{jk}}} \norm{\Sigma_{x_k}^{\frac{1}{2}}a_{xj}}_{2}}_{:=\eta^{j}_{k}(\kappa_{jk})}.
\end{align}
Similarly, we define Boolean quantities $\mathbf{g}^{j}_{k}$ representing the mean state being inside the tightened $j$\textsuperscript{th} constraint of $\calx$:
\begin{align}
    \mathbf{g}^{j}_{k} &= 
    \begin{cases}
    \top, \ \eqref{eqn_dr_constraint_tighten_env} \text{ is satisfied} \\ 
    \bot, \ \text{otherwise}.
    \end{cases} \label{eqn_gjk_boolean}
\end{align}
Here, $\eta^{j}_{k}$ is the deterministic constraint tightening of the $j$\textsuperscript{th} constraint of $\calx$ at time $k$ and \eqref{eqn_gjk_boolean} encodes the condition that the mean position of the robot should lie inside the tightened state constraint set $\calx$ in order to fulfill $\mathbf{g}^{j}_{k} = \top$.

\subsection{Exact Risk Allocation (ERA) Algorithm} \label{sec:era}
The aim of ERA is to allocate as little risks $\delta_{ik}$ and $\kappa_{jk}$ as possible for all obstacles $i \in \mathcal{B}$ and the constraints defining the state constraint set $\calx$ at all time steps $k$ that fulfill the DR risk constraint in \eqref{cc_individual} and \eqref{cc_individual_env} respectively. Note that ERA cannot be done if the mean state $\hat{x}_k$ is either inside the obstacle or outside $\calx$ as such paths will be deemed as non-feasible. Hence, we define the ERA problem with Boolean conditions for each constraints defining the obstacle $i \in \mathcal{B}$ and the constraints defining the state constraint set $\calx$.  
\begin{problem} \label{prob_4}
Find the minimum risk $\delta_{ik}$ for $i \in \mathcal{B}$ at all time steps $k \in [1:T]$ such that $\mathbf{h}_{ik} = \top$, and the minimum risk $\kappa_{jk}$ for which $\mathbf{g}^{l}_{k} = \top$ for $l=[1:n_{e}]$.
\end{problem} 

The following theorem tells us how to obtain the required minimum risks $\delta_{ik}$ and $\kappa_{jk}$ from \eqref{eqn_dr_constraint_tightening} and \eqref{eqn_dr_constraint_tighten_env}, respectively.

\begin{theorem} \label{thm_1}
The minimum risk for obstacle $i \in \mathcal{B}$ satisfying $\mathbf{h}_{ik} = \top$, at all time steps $k = 1,\dots,T$ is obtained by setting $a^{\mathsf{T}}_{ikj} \hat{x}_{k} = a^{\mathsf{T}}_{ikj} \hat{c}_{ikj} + \gamma^{j}_{ik}$ and is given by
\begin{equation}
\label{eqn_min_risk}
    \delta^{\star}_{ik} = \left(1 + \left(\frac{a^{\mathsf{T}}_{ikj} \hat{x}_{k} - a^{\mathsf{T}}_{ikj}\hat{c}_{ikj}}{\norm{(\Sigma_{x_{k}} + \Sigma_{cjk})^{\frac{1}{2}}a_{ikj}}_2}\right)^2 \right)^{-1},
\end{equation}
and the minimum risk for the $j$\textsuperscript{th} constraint defining $\calx$ satisfying $\mathbf{g}^{j}_{k} = \top, j = 1,\dots,n_{e}$ at all time steps $k = 1,\dots,T$ is obtained by setting $a^{\mathsf{T}}_{xj} \hat{x}_{k} = a^{\mathsf{T}}_{xj} c_{xj} - \eta^{j}_{k}$ and 
\begin{equation}
\label{eqn_min_risk_env}
    \kappa^{\star}_{jk} = \left(1 + \left(\frac{a^{\mathsf{T}}_{xj} c_{xj} - a^{\mathsf{T}}_{xj} \hat{x}_{k}}{\norm{\Sigma_{x_{k}}^{\frac{1}{2}}a_{xj}}_2}\right)^2 \right)^{-1}.
\end{equation}
\end{theorem}

\begin{proof}
Since $\Sigma_k$ and $\Sigma_{cjk}$ are known constants and $\sqrt{ \frac{1-\delta_{ik}}{\delta_{ik}}}$ is a decreasing function of $\delta_{ik}$, rearranging \eqref{eqn_hjik_boolean} for the case of $\mathbf{h}^{j}_{ik} = \top, j = 1,\dots,n_{i}$ leads to the below risk bound, 
\begin{equation}
\label{delta_ik}
    \delta_{ik} \geq \underbrace{\left(1 + \left(\frac{a^{\mathsf{T}}_{ikj} \hat{x}_{k} - a^{\mathsf{T}}_{ikj}\hat{c}_{ikj}}{\norm{(\Sigma_{x_{k}} + \Sigma_{cjk})^{\frac{1}{2}}a_{ikj}}_2}\right)^2 \right)^{-1}}_{:= \delta^{\star}_{ik}}.
\end{equation}
Similarly, rearranging \eqref{eqn_gjk_boolean} for the case of $\mathbf{g}^{j}_{k} = \top, j = 1,\dots,n_{e}$ leads to the following individual risk lower bound, 
\begin{equation}
\label{kappa_jk}
    \kappa_{jk} \geq \underbrace{\left(1 + \left(\frac{a^{\mathsf{T}}_{xj} c_{xj} - a^{\mathsf{T}}_{xj} \hat{x}_{k}}{\norm{\Sigma_{x_{k}}^{\frac{1}{2}}a_{xj}}_2} \right)^2 \right)^{-1}}_{:= \kappa^{\star}_{jk}}.
\end{equation}
\end{proof}

\section{Distributionally Robust RRT with Exact Risk Allocation} \label{sec_dr_rrt_era}
In this section, we extend the sampling-based Distributionally Robust RRT (DR-RRT) algorithm in \cite{summers_risk} which grows trees of state distributions while enforcing DR risk constraints, using the proposed ERA algorithm. Usually, DR-RRT employs the URA as it trivially satisfies \eqref{eqn_risk_sum}, where each obstacle and time step are first assigned the same risk $\delta_{ik} = \frac{\Delta}{TN}$, and the assigned risks are then used to check the probabilistic feasibility of the generated path according to the constraint in \eqref{eqn_hjik_boolean}. With ERA, the problem is tackled in the opposite way by \emph{first} assigning risks $\delta_{ik}$ that fulfill the DR risk constraints in \eqref{eqn_hjik_boolean} and \emph{then} checking if \eqref{eqn_risk_sum} holds.

\subsection{Tree Expansion}
Algorithm \ref{alg:tree} outlines the DR-RRT tree expansion with Exact Risk Allocation incorporated and the readers are referred to \cite{summers_risk} for information on DR-RRT tree expansion. Note that the trajectory generated from the LQR finite horizon steering function does not depend on the risk allocations $\delta_{ik}$. In the next step, Exact Risk Allocation is applied to the generated trajectory, as outlined in Algorithm \ref{alg:ERA}. The ERA-function returns risk allocations $\delta_{ik}$ and $\kappa_{jk}$ for all obstacles $i \in \mathcal{B}$ and all the constraints $j = 1,\dots,n_{e}$ defining the state constraint set $\calx$ at all time steps $k$ along the trajectory. The risk allocation is done so that \eqref{eqn_dr_constraint_tightening} and \eqref{eqn_dr_constraint_tighten_env} are fulfilled and the total risk leading up to each time step is obtained by summing up all risk allocations $\delta_{ik}$ up to a certain time step, denoted as $k^{\star}$. The path from $\mathrm{N_{near}}$ up to time step $\mathrm{T_{steer}}$ is then checked for distributionally robust feasibility, as outlined in Algorithm \ref{alg:fea}. If the path is feasible, the total cost $J$ and the residual risk $\mathrm{\delta_{res}}$ are calculated and used to assign a score to the path from the near node $\mathrm{N_{near}}$ as $\mathrm{score(N_{near})} = (\theta_{J}/J) + \mathrm{\theta_{res}} \mathrm{\delta_{res}}$, where $\theta_{J}, \mathrm{\theta_{res}} \in [0,1], \theta_{J} + \mathrm{\theta_{res}} = 1$ are left to the user's choice to emphasize the cost and the residual risk appropriately. When paths from all near nodes that are DR-feasible have been assigned a score, the path with the best score is chosen and a new node and edge is added to the tree. The residual risk $\mathrm{\delta_{res}}$ is also added to the node, which can in turn be re-allocated as described in subsection \ref{sec:fea} when steering from this node to a new sample. Feasible portions of the trajectories are also added to the tree in the same manner. 

\subsection{Feasibility Check} \label{sec:fea}
The feasibility check is based on the total risk allocated up to time step $k$, denoted by $\mathrm{\delta_{tot}}(k)$. The risk constraints \eqref{cc_individual}--\eqref{sumsum} have to hold for the entire planning horizon $T$ and not just over the steering horizons $\mathrm{T_{steer}} \in \mathbb{N}_{\geq 1}$. To assure this is the case, we begin by distributing the total risk budget $\Delta$ uniformly over all steering horizons according to $\mathrm{\Delta_{steer}}= \frac{\Delta \cdot \mathrm{T_{steer}}}{T}$, where $\mathrm{\Delta_{steer}}$ is the risk budget over each steering horizon $\mathrm{T_{steer}}$. An entire trajectory from a near node to the sample is deemed to be feasible, provided the total risk allocated over the steering horizon, $\mathrm{\delta_{tot}}(\mathrm{T_{steer}}) \leq \mathrm{\Delta_{steer}}$. A similar reasoning can be applied to assure the feasibility of a portion of the steered path, from a near node up to a certain time step $k$. Then, the total risk allocated up to that time step, $\mathrm{\delta_{tot}}(k)$, has to fulfill $\mathrm{\delta_{tot}}(k) \le \Delta_{k}$, where $\Delta_{k} := \frac{k \cdot \mathrm{\Delta_{steer}}}{\mathrm{T_{steer}}}$ is the uniformly allocated risk budget up to time step $k$. This means that a trajectory, or a portion of it, is considered feasible only when the total allocated risk (using ERA) does not exceed the corresponding total uniformly allocated risk. While this method has less conservatism than URA, there are still a lot of conservatism present from allocating the total risk budget uniformly over all steering horizons. This conservatism can be mitigated by reallocating residual risk of a horizon to the subsequent steering. If the entire risk budget $\mathrm{\Delta_{steer}}$ or $\Delta_k$  is not used, such that $\mathrm{\delta_{tot}}(\mathrm{T_{steer}}) < \mathrm{\Delta_{steer}}$ or $\mathrm{\delta_{tot}}(k) < \Delta_k$, a residual for the newly generated node at time step $k$ or $\mathrm{T_{steer}}$ can be created according to
\begin{align}
\label{aaaa}
    \mathrm{\delta_{res}} &= \mathrm{\Delta_{steer}} - \mathrm{\delta_{tot}}(\mathrm{T_{steer}})  \text{ or}\\
    \mathrm{\delta_{res}} &= \Delta_k - \mathrm{\delta_{tot}}(k). 
\end{align}    
These residual risks can then be reallocated to new trajectories generated from this node. When a new point $x_{s}$ is sampled, the residual risk of the near node $\mathrm{\delta_{res}}[\mathrm{N_{near}}]$ can be allocated to the trajectory generated by steering from $\mathrm{N_{near}}$ to $x_{s}$. The total risk budget for the new trajectory or its portion is then $\mathrm{\Delta_{steer}} + \mathrm{\delta_{res}}[\mathrm{N_{near}}]$ or $\Delta_k + \mathrm{\delta_{res}}[\mathrm{N_{near}}]$, respectively. Then, the feasibility of the trajectory generated from $\mathrm{N_{near}}$ depends upon the relaxed risk budget constraints 
\begin{align}
\label{cc_k_res}
    \mathrm{\delta_{tot}}(\mathrm{T_{steer}}) &\le \mathrm{\Delta_{steer}} + \mathrm{\delta_{res}}[\mathrm{N_{near}}] \text{ or}\\
    \mathrm{\delta_{tot}}(k) &\le \Delta_{k} + \mathrm{\delta_{res}}[\mathrm{N_{near}}], \label{cc_relaxed}
\end{align}
and the residual of $\mathrm{N_{near}}$ is added to the residual of newly created nodes originating from $\mathrm{N_{near}}$. That is, 
\begin{align}
    \mathrm{\delta_{res}} &= \mathrm{\Delta_{steer}} + \mathrm{\delta_{res}}[\mathrm{N_{near}}]- \mathrm{\delta_{tot}}(\mathrm{T_{steer}}) \text{ or} \\
    \mathrm{\delta_{res}} &= \Delta_k + \mathrm{\delta_{res}}[\mathrm{N_{near}}]- \mathrm{\delta_{tot}}(k). 
\end{align}
\noindent \textbf{Remarks:} Note that the above risk allocation procedure still has some conservatism. An inevitable conservatism stems from the usage of Boole's inequality to decompose the joint  risk constraint in \eqref{cc_joint} into individual risk constraints. Though some trajectories are deemed to be infeasible and dismissed as they fail to satisfy \eqref{cc_k_res}, they could be potentially stored with the hope that they become feasible when they are connected with new trajectories such that the \emph{combination} of the trajectories becomes risk-feasible. Albeit, such an effort would come at the expense of increased computational burden and memory storage along with the book-keeping to correctly identify feasible branches as near nodes to a random sample during the RRT tree expansion. For the ease of exposition, we decided not to implement the above mentioned aspects and only reallocate risk to future horizons. Interested readers are referred to \cite{kajsa_ms_thesis} for additional details.
\begin{theorem} \label{thm_2}
All DR-RRT paths feasible with the URA of length $\mathrm{T_{path}}$ and total risk budget $\mathrm{\Delta_{path}} \in (0,0.5]$ is also feasible with the ERA but the opposite is not necessarily true. 
\end{theorem}
\begin{proof}
Without loss of generality, we present the proof assuming that the environmental borders given by \eqref{eqn_environment_polytope} are not treated probabilistically. Consider a path that is feasible with URA meaning that all risk allocations are assigned the same value $\delta_{\mathrm{uni}} = \frac{\mathrm{\Delta_{path}}}{N \cdot \mathrm{T_{path}}}$ and they satisfy \eqref{eqn_risk_sum} and 
\begin{equation}
    \vee_{j=1}^{n_i} \, (a^{\mathsf{T}}_{ij} \hat{x}_{k} - a^{\mathsf{T}}_{ij} \hat{c}_{ikj} \geq \gamma^{j}_{ik}(\delta_{\mathrm{uni}})) , \ \, \substack{\forall i \in [1:N], \\\forall k \in [1:\mathrm{T_{path}}].} \nonumber
\end{equation}
On the other hand, ERA sets risk allocations $\delta_{ik}$ such that 
\begin{equation}
    \vee_{j=1}^{n_i} \, (a^{\mathsf{T}}_{ij} \hat{x}_{k} - a^{\mathsf{T}}_{ij} \hat{c}_{ikj} = \gamma^{j}_{ik}(\delta_{ik})), \ \, \substack{\forall i \in [1:N], \\\forall k \in [1:\mathrm{T_{path}}].} \nonumber   
\end{equation}
Then, $\forall i \in \mathcal{B}, \forall k = 1,\dots,\mathrm{T_{path}}$, and $j = 1,\dots,n_{i}$ 
\begin{align*}
    \gamma^{j}_{ik}(\delta_{\mathrm{uni}})) \leq \gamma^{j}_{ik}(\delta_{ik})) &\iff \sqrt{\frac{1-\delta_{\mathrm{uni}}}{\delta_{\mathrm{uni}}}} \leq \sqrt{\frac{1-\delta_{ik}}{\delta_{ik}}} \\
    &\iff \delta_{ik} \leq \delta_{\mathrm{uni}}. 
\end{align*}
Further, the sum of all exact risk allocations satisfies
\begin{align*}
    \sum_{k=1}^{\mathrm{T_{path}}} \sum_{i=1}^N \delta_{ik} \leq \sum_{k=1}^{\mathrm{T_{path}}} \sum_{i=1}^N \delta_{\mathrm{uni}} = \sum_{k=1}^{\mathrm{T_{path}}} \sum_{i=1}^N \frac{\mathrm{\Delta_{path}}}{N \cdot \mathrm{T_{path}}} = \mathrm{\Delta_{path}}. \nonumber
\end{align*}
Thus, the path is also feasible with the ERA. Conversely, we just need to prove that there exists a path which is feasible with ERA but not with URA. Suppose that at time step $k, \exists i \in \mathcal{B}$ such that the given path with risk allocation $\delta_{ik} < \delta_{\mathrm{uni}}$ is feasible with ERA meaning that $\sum_{k=1}^{\mathrm{T_{path}}} \sum_{i=1}^N \delta_{ik} = \mathrm{\bar{\Delta}_{path}} < \mathrm{\Delta_{path}}$. Since, $\sum_{k=1}^{\mathrm{T_{path}}} \sum_{i=1}^N \delta_{\mathrm{uni}} = \mathrm{\Delta_{path}} > \mathrm{\bar{\Delta}_{path}}$, we can conclude that $\exists j = 1,\dots,n_{i}$ for which $\gamma^{j}_{ik}(\delta_{\mathrm{uni}}) < \gamma^{j}_{ik}(\delta_{ik})$, and the path will be deemed infeasible with URA as both \eqref{eqn_hik_boolean} with $\gamma^{j}_{ik}(\delta_{\mathrm{uni}})$ and the risk summation condition that $\sum_{k=1}^{\mathrm{T_{path}}} \sum_{i=1}^N \delta_{\mathrm{uni}} \leq \mathrm{\Delta_{path}}$ cannot hold true simultaneously. Hence, there exist paths that are feasible with ERA but not with URA.
\end{proof}

\begin{algorithm}
\caption{\texttt{DR-RRT With ERA: Tree Expansion}}\label{alg:tree}
\begin{algorithmic}
\STATE \emph{Inputs:} Tree $\mathcal{T}$, time $k$, $\mathrm{T_{steer}}$, $\theta_{J}, \mathrm{\theta_{res}} \in [0,1]$ 
\STATE $x_s = $ sample$(\mathcal{X}^{\mathrm{free}}_{k})$
\STATE $\mathrm{N_{near}} = $ NearestMNodes$(x_s,\mathcal{T}, M)$
\FORALL{$\mathrm{N_{near}}$}
  \STATE $(\mathrm{\hat{x}_{path}}, \mathrm{\Sigma_{path}}) = $ Steer$(\mathrm{N_{near}}, x_s, \mathrm{T_{steer}})$
  \STATE $\delta, \kappa = $ ExactRiskAllocation$(\mathrm{\hat{x}_{path}},\mathrm{\Sigma_{path}}, \mathrm{T_{steer}})$
  \STATE $\mathrm{\delta_{tot}}({k^{\star}}) = \sum_{k=1}^{k^{\star}} (\sum_{i=1}^N \delta_{ik} + \sum_{j=1}^{n_{e}} \kappa_{jk}), \, \forall k^{\star} \leq \mathrm{T_{steer}}$
    \IF{DRFeasible$(\mathrm{\delta_{tot}}(\mathrm{T_{steer}}), \mathrm{\delta_{res}}[\mathrm{N_{near}}])$}
    \STATE $J = J[\mathrm{N_{near}}] + J(\mathrm{\hat{x}_{path}}, \mathrm{\Sigma_{path}})$
    \STATE $\mathrm{\delta_{res}}= \mathrm{\delta_{res}}[\mathrm{N_{near}}] + \mathrm{\Delta_{steer}} - \mathrm{\delta_{tot}}(\mathrm{T_{steer}})$
    \STATE score$(\mathrm{N_{near}}) = (\theta_{J}/J) + \mathrm{\theta_{res}} \mathrm{\delta_{res}}$
    \ENDIF
\ENDFOR
\end{algorithmic}
\begin{algorithmic}
\STATE \textbf{Select} path $(\mathrm{\hat{x}_{path}}, \mathrm{\Sigma_{path}})$ from $\mathrm{N_{near}}$ with best score
\STATE $\mathcal{T}$.AddNode$(\mathrm{\hat{x}_{path}}(\mathrm{T_{steer}}), \mathrm{\Sigma_{path}}(\mathrm{T_{steer}}), \mathrm{\delta_{res}})$
\STATE $\mathcal{T}$.AddEdge$(\mathrm{N_{near}},\mathrm{\hat{x}_{path}}(\mathrm{T_{steer}}))$
      \FOR{$k=1: \mathrm{T_{steer}}-1$}
        \IF{DRFeasible$(\mathrm{\delta_{tot}}(k), \mathrm{\delta_{res}}[\mathrm{N_{near}}])$}
           \STATE $\mathrm{\delta_{res}} = \mathrm{\delta_{res}}[\mathrm{N_{near}}] + \Delta_{k} - \mathrm{\delta_{tot}}(k)$
           \STATE $\mathcal{T}$.AddNode$(\mathrm{\hat{x}_{path}}(k), \mathrm{\Sigma_{path}}(k), \mathrm{\delta_{res}})$
           \STATE $\mathcal{T}$.AddEdge$(\mathrm{N_{near}},\mathrm{\hat{x}_{path}}(k))$
       \ENDIF
       \ENDFOR
\end{algorithmic}
\end{algorithm}

\begin{algorithm}
\caption{\texttt{ExactRiskAllocation}}\label{alg:ERA}
\begin{algorithmic}
\STATE \emph{Inputs:} Path $\mathrm{\hat{x}_{path}}, \mathrm{\Sigma_{path}}, \mathrm{T_{steer}}$
\STATE \emph{Output:} Risk matrices $\delta \in \mathbb{R}^{N \times \mathrm{T_{steer}}}, \kappa \in \mathbb{R}^{n_{e} \times \mathrm{T_{steer}}}$
\FOR{$k = 1: \mathrm{T_{steer}}$}
  \FOR{$i = 1: N$}
    \STATE Assign $\delta_{ik}$ satisfying \eqref{eqn_min_risk}.
  \ENDFOR
  \FOR{$j = 1: n_{e}$}
    \STATE Assign $\kappa_{jk}$ satisfying \eqref{eqn_min_risk_env}.
  \ENDFOR
\ENDFOR
\STATE $\texttt{return} \, \, \delta, \kappa $
\end{algorithmic}
\end{algorithm}

\begin{algorithm}[H]
\caption{\texttt{DRFeasible}}\label{alg:fea}
\begin{algorithmic}
\STATE \emph{Inputs:} total risk $\mathrm{\delta_{tot}}(k)$, residual of $\mathrm{N_{near}}$, $\mathrm{\delta_{res}}[\mathrm{N_{near}}]$
\STATE \emph{Output:} true if DR-feasible, otherwise false
\IF{$\mathrm{\delta_{tot}}(k)$ satisfies (\ref{cc_relaxed})}
\STATE $\texttt{return}$ \, true
\ELSE
\STATE $\texttt{return}$ \, false
\ENDIF
\end{algorithmic}
\end{algorithm}

\section{Numerical Simulations} \label{sec_num_sim}
\begin{figure*}
\begin{minipage}[t]{0.29\textwidth}
  \includegraphics[scale=0.085]{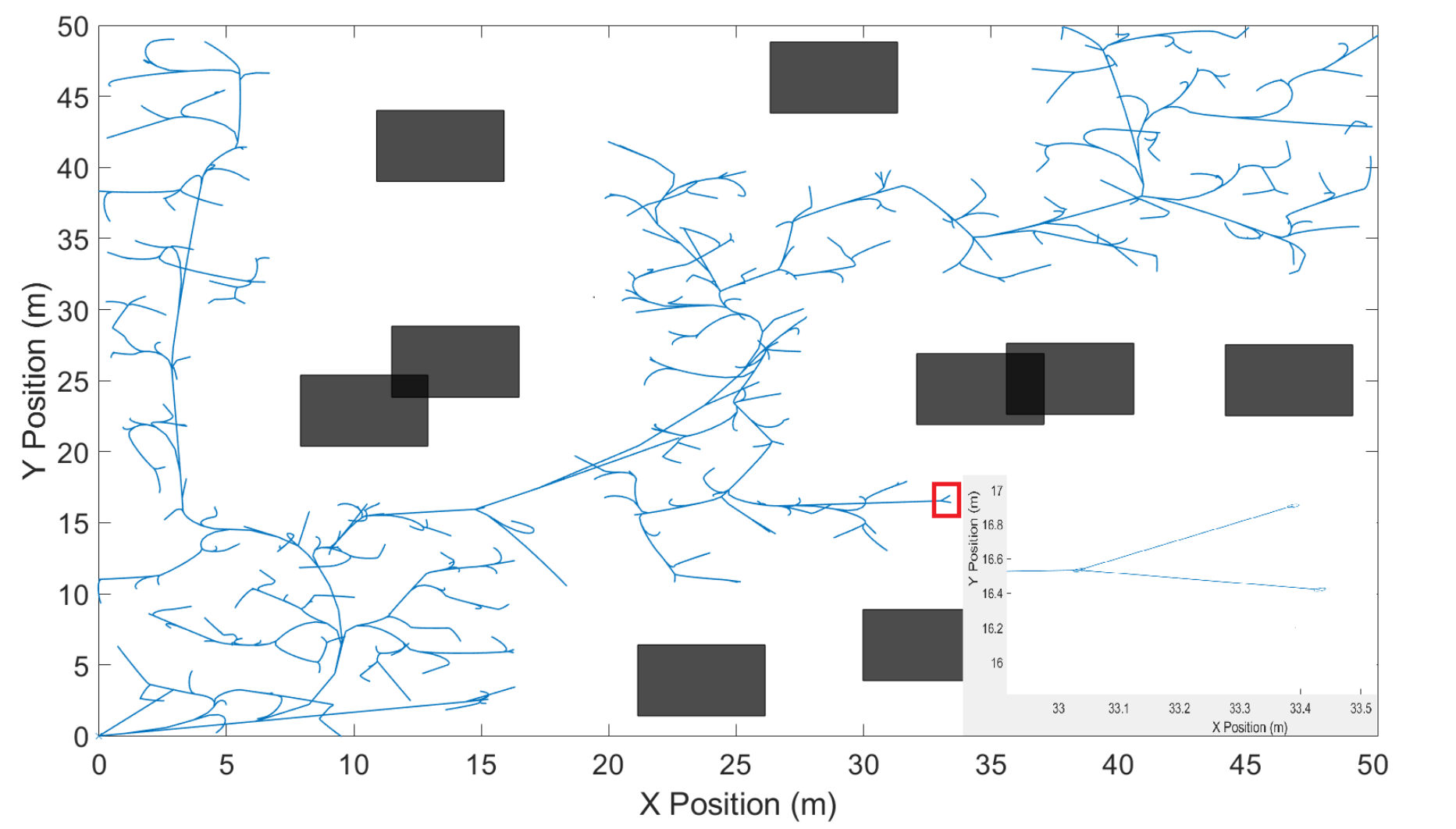}
  \caption{DR-RRT tree in $\mathbb{R}^{2}$ with URA using $\Delta = 0.1$ along with zoomed-in covariances.}
  \label{fig:first}
\end{minipage}%
\hfill 
\begin{minipage}[t]{0.31\textwidth}
  \includegraphics[scale=0.06]{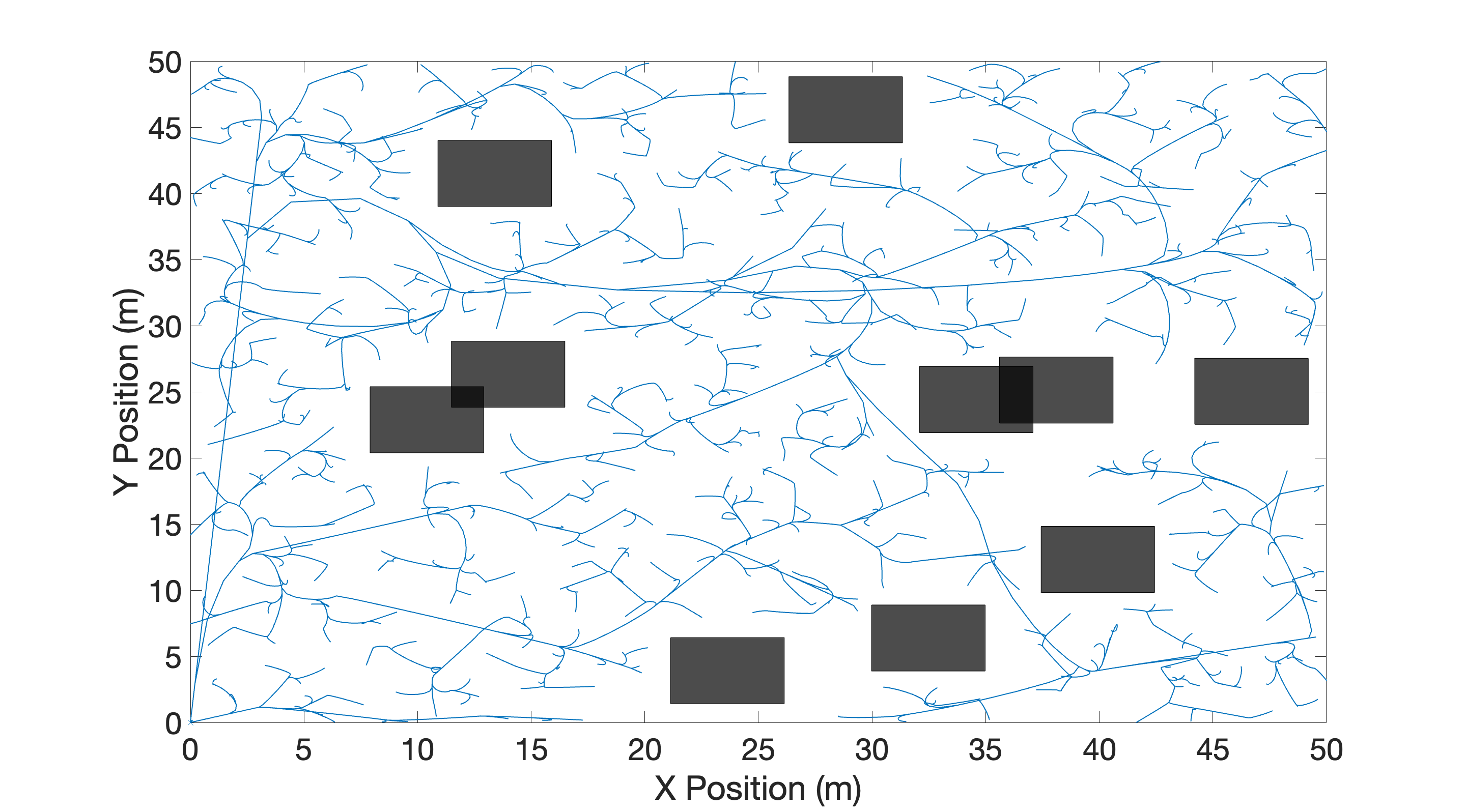}
  \caption{DR-RRT tree in $\mathbb{R}^{2}$ with ERA using $\Delta = 0.1$.}
  \label{fig:second}
\end{minipage}%
\hfill
\begin{minipage}[t]{0.31\textwidth}
  \includegraphics[scale=0.06]{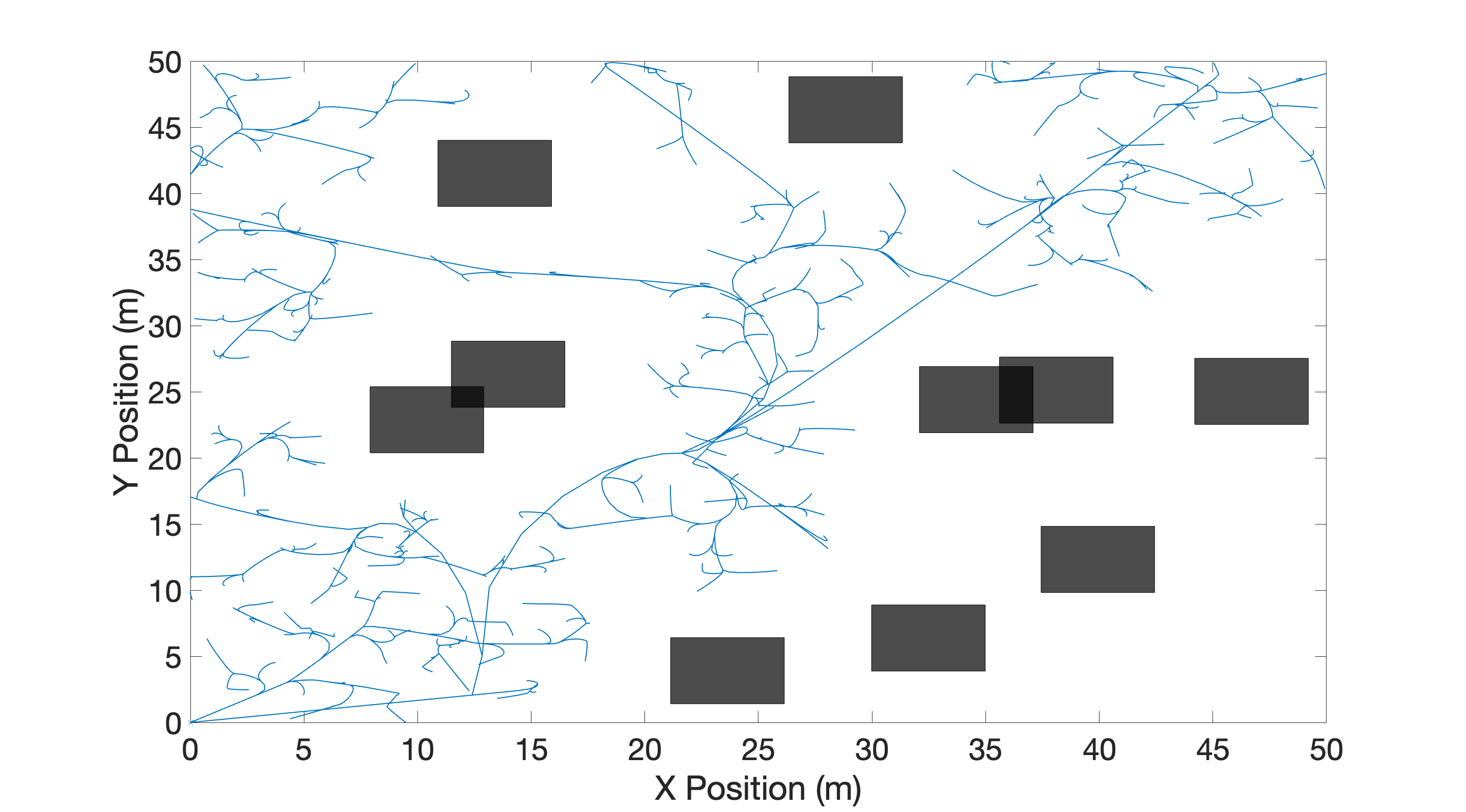}
  \caption{DR-RRT tree in $\mathbb{R}^{2}$ with ERA using $\Delta = 0.02$.}
  \label{fig:third}
\end{minipage}%
\end{figure*}
The environment for the simulation is a square area $[0,50]^{2}$~m$^2$ where $N=10$ rectangular obstacles are randomly placed. As for the simulations of DR-RRT in \cite{summers_risk}, a robot with discrete-time stochastic double-integrator dynamics having a mass of $1$~kg is considered. The initial position is $[0,0]$~m and the initial velocity is zero. The dynamics and input matrices are
\begin{align}
A =
\begin{bmatrix} 
I_{2} & dt I_{2} \\
0_{2 \times 2} & I_{2} 
\end{bmatrix}, &&
B = \begin{bmatrix} \frac{dt^2}{2} I_{2} \\ dt I_{2} \end{bmatrix},
\end{align}
with $dt=0.1$~s. The robot state is the position and velocity along each axis, with the corresponding force as inputs. The covariance matrices of the initial state and the disturbance are chosen as in \cite{summers_risk},
\begin{equation*} 
\footnotesize
\Sigma_{x_0} = 10^{-3}
\begin{bmatrix} 
1 & 0 & 0 & 0 \\
0 & 1 & 0 & 0 \\
0 & 0 & 0 & 0 \\
0 & 0 & 0 & 0
\end{bmatrix}, 
\Sigma_w = 10^{-3}
\begin{bmatrix} 
0 & 0 & 0 & 0 \\
0 & 0 & 0 & 0 \\
0 & 0 & 2 & 1 \\
0 & 0 & 1 & 2
\end{bmatrix}.
\end{equation*}
All obstacles are static and treated as deterministic, so that all uncertainty comes from the unknown state of the robot. The robot is treated as a point mass and the bounds on the environment are not treated probabilistically. As in \cite{summers_risk}, the steering from a near node to a sample $x_s$ is done by solving a discrete-time linear quadratic optimal control problem to compute the affine state feedback policy that minimises
\begin{gather}
    \sum_{k=0}^{T_{s}-1} \hat{e}_k^\mathsf{T} Q \hat{e}_k + u_k^\mathsf{T} R u_k + \hat{e}_{T_{s}}^\mathsf{T} Q \hat{e}_{T_{s}},
\end{gather}
where $\hat{e}_k = \hat{x}_k - x_s$ and $T_s = \mathrm{T_{steer}}$, $Q = 40 I_{4}$ and $R = 0.1$. The quadratic optimal cost-to-go function is also used as the distance metric in the selection of the nearest tree nodes. In all simulations, the trajectories of the mean state $\hat{x}_k$ is represented by lines and the uncertainty is represented by ellipses of one standard deviation, derived from the covariance $\Sigma_{x_k}$. Note that in all the simulations, the ellipses are in the range of $0.01$~m (visibly too small). The planning horizon is $T = 1000$ and the steering horizon is $\mathrm{T_{steer}}= 10$. The risk budget for the entire planning horizon $T$ is denoted as $\Delta$. Three DR-RRT trees with 1000 samplings are simulated, namely: 1) Using URA and risk budget $\Delta = 0.1$, 2) Using ERA and risk budget $\Delta = 0.1$, and 3) Using ERA and risk budget $\Delta = 0.02$. The risk allocation of ERA was done using results from Theorem \ref{thm_1}. Besides from the risk allocation and risk budget, everything in the trees and environment are exactly the same, including the random sampling points. This is to get a fair comparison of the different trees. With URA, the same risk is allocated for all obstacles and time steps, such that $\delta_{ik} = \frac{\Delta}{T \cdot N} = \frac{0.1}{1000 \cdot 10} = 10 ^{-5}$. With ERA, the risk budget for a steering horizon is $\mathrm{\Delta_{steer}}+ \mathrm{\delta_{res}}$, where $\mathrm{\delta_{res}}$ is the residual of the node from which steering is done and 
\begin{equation*}
    \footnotesize
    \mathrm{\Delta_{steer}} = \frac{\Delta \cdot \mathrm{T_{steer}}}{T} = \begin{cases} 10 ^{-3}, &\text{ when } \Delta = 0.1 \\ 2 \times 10^{-4}, &\text{ when } \Delta = 0.02. \end{cases} 
\end{equation*}

\begin{table}[h]
\centering
\caption{Average Results of 1000 Independent Simulations each with 1000 iterations}
\label{tab_sim_results}
\begin{tabular}{|c|c|c|}
\hline
Methodology & $\Delta$ & \# Nodes $(|\mathcal{T}|)$ \\
\hline \hline
URA & $0.10$ & 3101 \\
\hline
ERA & $0.10$ & 8175 \\
\hline
ERA & $0.02$ & 3348 \\
\hline
\end{tabular}
\end{table}

\noindent \textbf{Discussion:} Figs. \ref{fig:first} and \ref{fig:second} show that DR-RRT with ERA generates less conservative paths than DR-RRT with uniform risk allocation when the same risk budget $\Delta = 0.1$ is used. With the same $\Delta$, DR-RRT with ERA explores the state-space more efficiently than with URA (consequence of Theorem \ref{thm_2}), still by having the same risk guarantees. Fig. \ref{fig:third} illustrates how DR-RRT with ERA can be used with lower risk budget $\Delta = 0.02$ and still generate paths with comparably a similar degree of conservatism as DR-RRT with a URA using a higher risk budget of $\Delta = 0.1$. The comparison in Table \ref{tab_sim_results} is a good indication of the above observation. In general, the selection of best $\Delta$ is not straightforward as it depends upon $T, \Sigma_{x_{0}}, \Sigma_{w}$ and the steering law being used. Overall, ERA gives the same risk guarantees for sampling-based motion planning algorithms, while maintaining a reduced conservatism and with almost no additional computational complexity resulting from the risk-allocation procedure.

\section{Conclusion} \label{sec_conclusion}
An extension of the sampling-based probabilistically complete DR-RRT motion planning algorithm in \cite{summers_risk} with an optimal risk allocation was presented. We proved that our risk allocation based embedding technique realises guaranteed conservative, yet increasingly more risk feasible trajectories for efficient state-space exploration. That is, all DR-RRT paths feasible with the URA are feasible with the ERA but not vice-versa. Future research will aim to design a slightly more involved risk allocation based embedding into the DR-RRT\textsuperscript{$\star$} algorithm from \cite{venki_risk, renganathan2023risk} to guarantee both risk-bounded and asymptotically optimal trajectories. 


\addtolength{\textheight}{-12.2cm}

\bibliographystyle{IEEEtran}
\bibliography{references}


\end{document}